\documentclass{article}

\PassOptionsToPackage{numbers, sort&compress}{natbib}

    \usepackage[preprint]{neurips_2020}

\usepackage[utf8]{inputenc} 
\usepackage[T1]{fontenc}    
\usepackage{hyperref}       
\usepackage{url}            
\usepackage{booktabs}       
\usepackage{amsfonts}       
\usepackage{nicefrac}       
\usepackage{microtype}      
\usepackage{xcolor}
\usepackage{amsthm}
\usepackage{mathtools}
\usepackage{amsmath}
\usepackage{amssymb}
\usepackage{graphicx}
\usepackage{caption}
\usepackage{subcaption}
\usepackage{capt-of}
\usepackage{thmtools}
\usepackage{thm-restate}
\usepackage{multirow}
\usepackage[ruled,vlined]{algorithm2e}

\DeclareMathOperator*{\argmin}{arg\,min}
\DeclareMathOperator{\E}{\mathbb{E}}

\DeclarePairedDelimiter\abs{\lvert}{\rvert}%
\DeclarePairedDelimiter\norm{\lVert}{\rVert}%

\makeatletter
\let\oldabs\abs
\def\abs{\@ifstar{\oldabs}{\oldabs*}}
\let\oldnorm\norm
\def\norm{\@ifstar{\oldnorm}{\oldnorm*}}
\makeatother

\newtheorem{lemma}{Lemma}

\theoremstyle{remark}

\theoremstyle{definition}
\newtheorem{definition}{Definition}

\title{Joint Contrastive Learning for \\Unsupervised Domain Adaptation}

\author{%
  Changhwa Park$^1$\quad Jonghyun Lee$^1$\quad Jaeyoon Yoo$^1$\quad Minhoe Hur$^2$\quad Sungroh Yoon$^{1,}$\thanks{Corresponding author: Sungroh Yoon <sryoon@snu.ac.kr>.}\\
  $^1$Data Science \& Artificial Intelligence Lab, Seoul National University\\
  $^2$Artificial Intelligence Research Lab, Hyundai Motor Group\\
  \texttt{omega6464@snu.ac.kr, leejh9611@gmail.com, yjy765@snu.ac.kr} \\
  \texttt{minhoe.hur@hyundai.com, sryoon@snu.ac.kr} \\
}

\begin{document}

\maketitle

\begin{abstract}
Enhancing feature transferability by matching marginal distributions has led to improvements in domain adaptation, although this is at the expense of feature discrimination.
In particular, the ideal joint hypothesis error in the target error upper bound, which was previously considered to be minute, has been found to be significant, impairing its theoretical guarantee.
In this paper, we propose an alternative upper bound on the target error that explicitly considers the joint error to render it more manageable.
With the theoretical analysis, we suggest a joint optimization framework that combines the source and target domains.
Further, we introduce Joint Contrastive Learning (JCL) to find class-level discriminative features, which is essential for minimizing the joint error.
With a solid theoretical framework, JCL employs contrastive loss to maximize the mutual information between a feature and its label, which is equivalent to maximizing the Jensen-Shannon divergence between conditional distributions.
Experiments on two real-world datasets demonstrate that JCL outperforms the state-of-the-art methods.
\end{abstract}

\section{Introduction}
\label{sec:introduction}
Deep neural networks have been successfully adopted in numerous applications~\cite{lecun2015deep, redmon2016you} and show promising performance.
In practice, however, collecting a large volume of labeled data for a new target domain is often expensive or even impractical and has become a considerable impediment to the application of deep learning algorithms.
To circumvent this problem, domain adaptation has been introduced~\cite{pan2009survey}, which utilizes labeled data from a source domain to classify the target domain.

The major characteristic of domain adaptation is the dataset shift~\cite{quionero2009dataset, gretton2009covariate} that precludes small target error when the classifier trained on the source domain is directly applied to the unlabeled target data.
A theoretical analysis provided by~\cite{ben2010theory} indicates that the target error is upper bounded by the sum of the source error, the domain discrepancy, and the error of the ideal joint hypothesis, while the last term is often treated as constant in the literature.
Several studies~\cite{tzeng2014deep, sun2016deep, pmlr-v37-ganin15, tzeng2017adversarial, pmlr-v37-long15, long2017deep} have focused on reducing the discrepancy between the marginal distributions of the domains in feature space.

Although matching the source and target feature distributions has led to increased accuracy as it enhances feature transferability, it has been at the expense of feature discriminability~\cite{chen2019transferability}.
To illustrate, methods based on marginal distribution alignment in the domain level may disregard class-conditional distributions.
As a result, it is possible that decision boundaries traverse high-density regions of the target domain, rendering the learned classifier vulnerable to misclassification~\cite{shu2018a, yoo2019learning}.

To enhance feature discriminability, recent domain adaptation methods have explored two strategies.
The first strategy matches the first-order statistics of conditional distributions~\cite{kang2019contrastive, deng2019cluster, xie2018learning}.
The second strategy uses pairwise loss~\cite{bromley1994signature, luo2018smooth} or triplet loss~\cite{schroff2015facenet} to learn discriminative features~\cite{deng2019cluster, chen2019joint, dai2019contrastively}.
There are two key issues with these methods.
First, reducing the distances between the source and target class centers can coarsely align the class-conditional distributions, but this is far from discriminative features.
Second, from an information-theoretic perspective, using pairwise loss or triplet loss cannot properly estimate and maximize the mutual information (MI) between a learned representation and its label, as discussed in Section~\ref{sec:4_2_1_Theoretical_guarantees}.
The MI between a feature and its label is equivalent to the Jensen-Shannon (JS) divergence between class-conditional distributions.
We therefore propose to maximize the MI, which theoretically guarantees the learning of class-wise discriminative features. 

In this paper, we first address the error of the ideal joint hypothesis, which has not been sufficiently investigated but can significantly affect the upper bound on the target error.
We suggest an alternative upper bound on the target error from the perspective of joint optimization to explicitly consider the joint hypothesis error.
The joint hypothesis error in the proposed upper bound is directly affected by the hypothesis, and hence, it is straightforward to implement.
Further, we propose a novel Joint Contrastive Learning (JCL) approach to unsupervised domain adaptation, which considers the union distribution of the source and target to minimize the proposed joint hypothesis error.
Within this framework, labeled data from the source domain and unlabeled data from the target domain are combined and jointly optimized to achieve a minimum joint error.
In particular, we enlarge the JS divergence between different class-conditional distributions of the combined dataset by maximizing the MI between a feature and its label using InfoNCE constrastive loss~\cite{oord2018representation}.
Experiments demonstrate that our proposed JCL achieves state-of-the-art results on several benchmark datasets.



\section{Preliminary}
\label{sec:preliminary}
We first introduce the setting of unsupervised domain adaptation and previous theoretical analysis~\cite{ben2010theory}.

\subsection{Problem setting and notations}
In an unsupervised domain adaptation framework, we have access to a set of labeled source data $\{(\boldsymbol{x}_s^i, y_s^i) \in (\mathcal{X} \times \mathcal{Y})\}_{i=1}^n$, sampled i.i.d. from the source domain distribution $\mathcal{D}_S$ and a set of unlabeled target data $\{\boldsymbol{x}_t^j \in \mathcal{X}\}_{j=1}^m$ sampled i.i.d. from the target domain distribution $\mathcal{D}_T$.
A domain is defined as a pair comprised of a distribution $\mathcal{D}$ in the input space $\mathcal{X}$ and a labeling function $f: \mathcal{X} \rightarrow \mathcal{Y}$, where the output space $\mathcal{Y}$ is $[0, 1]$ in the theoretical analysis.
Consequently, we use $\langle \mathcal{D}_S, f_S \rangle$ and $\langle \mathcal{D}_T, f_T \rangle$ to denote the source and target domains, respectively.
The goal of unsupervised domain adaptation is to learn a hypothesis function $h: \mathcal{X} \mapsto \mathcal{Y}$ that provides a good generalization in the target domain.
Formally, the error of a hypothesis $h$ with respect to a labeling function $f_T$ under the target domain distribution $\mathcal{D}_T$ is defined as $\epsilon_T(h, f_T) \coloneqq \mathbb{E}_{\boldsymbol{x} \sim \mathcal{D}_T}[\abs{h(\boldsymbol{x}) - f_T(\boldsymbol{x})}]$.

\subsection{Theoretical background for domain adaptation}
We review the theoretical basis of domain adaptation~\cite{ben2010theory}.
The ideal joint hypothesis is defined as,

\begin{definition} [\cite{ben2010theory}]
Let $\mathcal{H}$ be a hypothesis space.
The \textit{ideal joint hypothesis} is the hypothesis which minimizes the combined error: $h^* \coloneqq \argmin_{h \in \mathcal{H}} \epsilon_S (h, f_S) + \epsilon_T (h, f_T)$. We denote the combined error of the ideal hypothesis by $\lambda \coloneqq \epsilon_S (h^*, f_S) + \epsilon_T (h^*, f_T)$.
\end{definition}

\citet{ben2010theory} proposed the following theorem for the upper bound on the target error, which is used as the theoretical background for numerous unsupervised domain adaptation methods.

\begin{restatable}[\cite{ben2010theory}]{thm}{bendavid}
\label{bendavid_theorem}
Let $\mathcal{H}$ be a hypothesis space, then the expected target error is upper bounded as,
\[\epsilon _{T}(h, f_T) \leq \epsilon_{S}(h, f_S) + \frac{1}{2}d_{\mathcal{H}\Delta\mathcal{H}}(\mathcal{D}_S, \mathcal{D}_T) + \lambda,\]
where $d_{\mathcal{H}\Delta\mathcal{H}}(\mathcal{D}_S, \mathcal{D}_T)$ is $\mathcal{H}\Delta\mathcal{H}$-distance between the source and target distributions.
Formally, 
\[d_{\mathcal{H}\Delta\mathcal{H}}(\mathcal{D}_S, \mathcal{D}_T) \coloneqq 2 \sup_{h, h^\prime \in \mathcal{H}} \abs{\mathrm{Pr}_{\boldsymbol{x} \sim \mathcal{D}_S} [h(\boldsymbol{x}) \neq h^\prime(\boldsymbol{x})] - \mathrm{Pr}_{\boldsymbol{x} \sim \mathcal{D}_T} [h(\boldsymbol{x}) \neq h^\prime(\boldsymbol{x})]}.\]
\end{restatable}

Theorem~\ref{bendavid_theorem} shows that the target error is upper bounded by the sum of the source error, $\mathcal{H}\Delta\mathcal{H}$-distance between the domains, and the error of the ideal joint hypothesis, $\lambda$.
The last term $\lambda$ is often considered to be minute.
Accordingly, many recent works on domain adaptation endeavored to allow a feature encoder $g: \mathcal{X} \mapsto \mathcal{Z}$ to learn such that the induced distributions of the domains in feature space $\mathcal{Z}$ have a minimal $\mathcal{H}\Delta\mathcal{H}$-distance, while also minimizing the source error~\cite{pmlr-v37-ganin15, JMLR:v17:15-239, NIPS2018_8075, pei2018multi}.
However, as demonstrated in~\cite{chen2019transferability}, $\lambda$ can become substantial as feature transferability is enhanced.
In particular, it was observed in~\cite{chen2019transferability} that the optimal joint error on the learned feature representation of Domain Adversarial Neural Network (DANN)~\cite{pmlr-v37-ganin15} is much higher than that of pre-trained ResNet-50~\cite{he2016deep}.
This suggests that learning domain-invariant features is not sufficient to tightly bound the target error.

\section{Related work}
\label{sec:related_work}


\paragraph{Marginal distribution alignment.}
As suggested by the theoretical analysis in~\cite{ben2010theory}, aligning the marginal distributions can result in reducing target error, and is common practice in domain adaptation~\cite{tzeng2014deep, sun2016deep, pmlr-v37-ganin15, tzeng2017adversarial, pmlr-v37-long15, long2017deep}.
\citet{pmlr-v37-long15, long2017deep} utilized the Maximum Mean Discrepancy (MMD) and Joint MMD, respectively, for estimating and minimizing the domain discrepancy over the domain-specific layers.
Inspired by Generative Adversarial Networks (GANs)~\cite{goodfellow2014generative}, \citet{pmlr-v37-ganin15} introduced a domain discriminator to convert the domain confusion into a minmax optimization.
However, \citet{shu2018a} and \citet{pmlr-v97-zhao19a} theoretically demonstrated that finding invariant representations is not sufficient to guarantee a small target error, and \citet{chen2019transferability} empirically revealed an unexpected deterioration in discriminability while learning transferable representations.

\paragraph{Discriminative representation learning.}
Researchers have attempted to learn class-level discriminative features using two main technologies: first-order statistics matching and Siamese-networks training~\cite{bromley1994signature, luo2018smooth, schroff2015facenet}.
\citet{xie2018learning} aligned the source and target centroids to learn semantic representations of the target data.
Utilizing the MMD measurement, \citet{kang2019contrastive} proposed the minimization of intra-class discrepancy and maximizing the inter-class margin.
\citet{deng2019cluster} employed pairwise margin loss to learn discriminative features and minimized the distances between the first-order statistics to align the conditional distributions.
However, there are two main differences between the proposed JCL approach and these methods.
First, with the theoretical analysis that explicitly handles the joint hypothesis error, we propose to jointly optimize the source and target domains to have class-wise discriminative representations.
Second, we theoretically guarantee learning discriminative features from the perspective of JS divergence by maximizing the MI between a feature and its label.
Previous pairwise loss or triplet loss-based methods cannot properly bound and maximize the MI.


\paragraph{Contrastive learning.}
Contrastive learning has been adopted for self-supervised learning and has led to significant performance enhancement.
\citet{oord2018representation} introduced \textit{InfoNCE} loss to estimate and maximize the MI between a present context and a future signal.
In the image domain, \citet{hjelm2018learning}, \citet{bachman2019learning}, and \citet{chen2020simple} maximized the MI between features that originated from the same input with different augmentations.
From the information-theoretic perspective, \citet{Tschannen2020On} suggested that these methods could be subsumed under the same objective, \textit{InfoMax}~\cite{linsker1988self} (see Section~\ref{sec:4_2_1_Theoretical_guarantees} for more details), and provided a different perspective on the success of these methods.
As opposed to these methods, the proposed approach maximizes the MI between features from the same class to maximize the JS divergence between different class-conditional distributions.

\section{Method}
\label{sec:method}
\subsection{An alternative upper bound}
\label{sec:4_1_Alternative_Error}
In the past, it was assumed that the ideal joint hypothesis error was insignificant, and therefore, it was neglected. 
However, recent studies~\cite{chen2019transferability, pmlr-v97-zhao19a} have suggested that this error can become substantial, and so must be addressed adequately.
Computations of the ideal joint hypothesis in Theorem~\ref{bendavid_theorem} are usually intractable, because of which addressing the optimal joint error is challenging.
With the proposed method, we therefore aim to provide an alternative upper bound on the target error, which clearly incorporates the concept of joint error, and is free from the optimal hypothesis.
A small ideal joint hypothesis error implies that there exists a joint hypothesis, which generalizes well on both the source and target domains.
Intuitively, it is natural to consider jointly optimizing within the domains to minimize the joint error.
From this point of view, we define a combined domain $\langle \mathcal{D}_U, f_U \rangle$ as below.

\begin{definition}
Let $\phi_S$ and $\phi_T$ be the density functions of the source and target distributions, respectively.
Then, the distribution of the combined domain $\mathcal{D}_U$ is the mean distribution of the source and target distributions: $\phi_U(\boldsymbol{x}) \coloneqq \frac{1}{2} (\phi_S(\boldsymbol{x}) + \phi_T(\boldsymbol{x}))$.
The labeling function $f_U$ of the combined domain is defined similarly: $f_U(\boldsymbol{x}) \coloneqq \frac{1}{2} (f_S(\boldsymbol{x}) + f_T(\boldsymbol{x}))$.
\end{definition}

With the definition of the combined domain, the following theorem holds:

\begin{restatable}{thm}{alternative}
\label{alternative_theorem}
Let $\mathcal{H}$ be a hypothesis space, then the expected target error is upper bounded as,
\[\epsilon_{T}(h, f_T) \leq \epsilon_{S}(h, f_S) + \frac{1}{4}d_{\mathcal{H}\Delta\mathcal{H}}(\mathcal{D}_S, \mathcal{D}_T) + 2 \epsilon_{U}(h, f_U).\]    
\end{restatable}

\begin{proof}
The proof is provided in Appendix A.
\end{proof}

\paragraph{Comparison with Theorem~\ref{bendavid_theorem}.}
The main difference between Theorem~\ref{bendavid_theorem} and Theorem~\ref{alternative_theorem} lies in $\lambda$ in Theorem~\ref{bendavid_theorem} and $2 \epsilon_{U}(h, f_U)$ in Theorem~\ref{alternative_theorem}.
To illustrate, $\lambda$ in Theorem~\ref{bendavid_theorem} is composed of the ideal joint hypothesis, which is neither tractable nor manageable, and hence, it has been obliquely addressed~\cite{deng2019cluster, chen2019progressive, chen2019joint, saito2017asymmetric}.
On the contrary, $2 \epsilon_{U}(h, f_U)$, the alternative term in Theorem~\ref{alternative_theorem}, is directly affected by the hypothesis $h$, and thus, it is straightforward to utilize.
Differ from the previous studies that attempt to only alter $\lambda$ from Theorem~\ref{bendavid_theorem}, Theorem~\ref{alternative_theorem} cannot be directly derived from Theorem~\ref{bendavid_theorem}, because the second term in Theorem~\ref{alternative_theorem} is smaller than that in Theorem~\ref{bendavid_theorem}.


The main idea here is that joint optimization in the source and target domains is demanded upon simply matching the marginal distributions of the domains.
As the target labels are not provided, we must rely on the source labels.
However, optimization of the source domain alone can result in poor generalization of the target domain.
We therefore combine the source and target domains and propose their joint optimization.
To estimate the joint hypothesis error, we resort to target pseudo-labels, and the following theorem holds:

\begin{restatable}{thm}{pseudo}
\label{pseudo_alternative_theorem}
Let $\mathcal{H}$ be a hypothesis space, and $f_{\hat{T}}$ be a target pseudo-labeling function.
Accordingly, $f_{\hat{U}}$ is defined as, $f_{\hat{U}}(\boldsymbol{x}) \coloneqq \frac{1}{2} (f_S(\boldsymbol{x}) + f_{\hat{T}}(\boldsymbol{x}))$.
Then the expected target error is upper bounded as, 
\[\epsilon_{T}(h, f_T) \leq \epsilon_{S}(h, f_S) + \frac{1}{4}d_{\mathcal{H}\Delta\mathcal{H}}(\mathcal{D}_S, \mathcal{D}_T) + 2 \epsilon_{U}(h, f_{\hat{U}}) + \epsilon_{T}(f_T, f_{\hat{T}}).\]
\end{restatable}

\begin{proof}
The proof is provided in Appendix A.
\end{proof}


\subsection{Joint Contrastive Learning}
\label{sec:4_2_Joint_Contrastive_Learning}
As the learned feature representation has a substantial impact on classification error, it is important to learn discriminative features to minimize this error.
For instance, in unsupervised learning, which is similar to unsupervised domain adaptation because the target true labels are not available, discriminative feature learning has brought remarkable progress in a variety of tasks~\cite{chen2020simple, he2019momentum, oord2018representation}.
At a high level, the main idea here is that segregating the features of dissimilar samples helps a classifier to generalize well without requiring high model complexity.
Our proposed method also utilizes the notion of learning discriminative feature representation in minimizing the joint hypothesis error.

\subsubsection{Theoretical guarantees}
\label{sec:4_2_1_Theoretical_guarantees}
Formally, we aim to learn discriminative features on the intermediate representation space $\mathcal{Z}$ induced through the feature transformation $g$.
We denote the induced distribution of the combined domain $\mathcal{D}_U$ over the representation space $\mathcal{Z}$ as $\mathcal{D}^{\mathcal{Z}}_{U}$, and its class-conditional distribution as $\mathcal{D}^{\mathcal{Z}}_{U|y}$, where $y$ is a class label.
We can then formalize our objective with JS divergence $D_{\mathrm{JS}}$ as follows:

\begin{equation}
    \max_{\theta_g} D_{\mathrm{JS}}(\mathcal{D}^{\mathcal{Z}}_{U|0} \| \mathcal{D}^{\mathcal{Z}}_{U|1}),
\end{equation}

where $\theta_g$ denotes the parameters of the feature encoder $g$.
The values 0 and 1 are the class labels, and hence, the objective means maximizing the divergence between different class-conditional distributions.
We consider binary classification for the theoretical analysis here, but this approach can also be generalized to a multiclass classification problem (see Appendix B for more details).

Suppose that the label distribution of the combined domain is uniform, i.e., $P(y=0) = P(y=1)$.
In practice, this can be achieved by reformulating a dataset to be class-wise uniform, as described in Section~\ref{sec:4_2_2_Training_procedure}.
Let $Y$ be a uniform random variable that takes the value in $\{0, 1\}$ and let the distribution $\mathcal{D}^{\mathcal{Z}}_{U|Y}$ be the mixture of $\mathcal{D}^{\mathcal{Z}}_{U|0}$ and $\mathcal{D}^{\mathcal{Z}}_{U|1}$, according to $Y$.
We denote the induced feature random variable with the distribution $\mathcal{D}^{\mathcal{Z}}_{U|Y}$ as $\boldsymbol{Z}_{U|Y}$.
From the relation between JS divergence and MI, $D_{\mathrm{JS}}(\mathcal{D}^{\mathcal{Z}}_{U|0} \| \mathcal{D}^{\mathcal{Z}}_{U|1}) = I(Y;\boldsymbol{Z}_{U|Y})$ holds.
Therefore, we can transform our objective as follows:

\begin{equation}
    \max_{\theta_g} I(Y;\boldsymbol{Z}_{U|Y}).
\end{equation}

The MI between a label and a feature that is induced from the distribution conditioned on the label can be maximized using the following approach.
We employ the InfoNCE loss proposed by~\citet{oord2018representation} to estimate and maximize the MI.
InfoNCE is defined as,

\begin{equation}
    I(X;Y) 
    \geq \E \Bigg[\frac{1}{K} \sum_{i=1}^K \log \frac{e^{c(x_i, y_i)}}{\frac{1}{K} \sum_{j=1}^K e^{c(x_i, y_j)}}\Bigg]
    \triangleq I_{\mathrm{NCE}}(X;Y),
\end{equation}

where the expectation is over $K$ independent samples from the joint distribution $p(x, y)$~\cite{pmlr-v97-poole19a}.
$c(x, y)$ is a \textit{critic} function used to predict whether the inputs $x$ and $y$ were jointly drawn by yielding high values for the jointly drawn pairs and low values for the others~\cite{Tschannen2020On}.

The proposed JCL framework does not directly pair a feature and its label to maximize the MI between them.
Instead, features from the same conditional distribution are paired, and we use $I_{\mathrm{NCE}}$ to maximize the MI between them.
For a given $Y$, let $\boldsymbol{Z}_{U|Y}^{(1)}$ and $\boldsymbol{Z}_{U|Y}^{(2)}$ be two different features that are sampled from the same conditional distribution $\mathcal{D}^{\mathcal{Z}}_{U|Y}$.
Then, $I(\boldsymbol{Z}_{U|Y}^{(1)}; \boldsymbol{Z}_{U|Y}^{(2)}) \leq I(Y;\boldsymbol{Z}_{U|Y}^{(1)}, \boldsymbol{Z}_{U|Y}^{(2)})$ holds, with the data processing inequality (see Appendix C for more details).
Therefore, $\max_{\theta_g} I(\boldsymbol{Z}_{U|Y}^{(1)}; \boldsymbol{Z}_{U|Y}^{(2)})$ can be seen as a lower bound for our objective $\max_{\theta_g} I(Y;\boldsymbol{Z}_{U|Y})$, and we optimize it with our InfoNCE loss $\mathcal{L}_{\mathrm{c}}$, as described below.


\paragraph{Comparison with InfoMax objective.}
Comparing our objective $\max_{\theta_g} I(Y;\boldsymbol{Z}_{U|Y})$ with the InfoMax objective $\max_{\theta_g} I(\boldsymbol{X};g(\boldsymbol{X}))$~\cite{linsker1988self} provides instructive insights.
Recent progress on unsupervised representation learning~\cite{chen2020simple, oord2018representation, hjelm2018learning, tian2019contrastive} can be subsumed under the same objective $\max_{\theta_{g_1}, \theta_{g_2}} I(g_1(\boldsymbol{X}^{(1)});g_2(\boldsymbol{X}^{(2)}))$, where $\boldsymbol{X}^{(1)}$ and $\boldsymbol{X}^{(2)}$ are instances that originate from the same data~\cite{Tschannen2020On}.
Using the process similar to that derived above, it can be shown that the objective is a lower bound on the InfoMax objective.
The main difference is that the InfoMax principle essentially aims to maximize the MI between data and its representation, whereas our objective focuses on maximizing the divergence between different class-conditional distributions in the feature space.

\paragraph{Comparison with triplet loss-based methods.}
The multi-class-K-pair loss~\cite{sohn2016improved}, which is the generalized triplet loss~\cite{weinberger2009distance}, can be shown to be a special case of InfoNCE loss~\cite{Tschannen2020On}, and triplet loss is the same as in the $K=2$ case.
The drawback of using triplet loss to learn discriminative features is that it cannot tightly bound the MI when the MI is larger than $\log K$ because $I_{\mathrm{NCE}}$ is upper bounded by $\log K$.
Pairwise margin loss also compares only two features, and hence, it is also expected to have a loose bound.
Thus, triplet loss or pairwise loss-based domain adaptation methods~\cite{deng2019cluster, chen2019joint, dai2019contrastively} cannot guarantee class-level discriminative features from an information-theoretic perspective.

\subsubsection{Training procedure}
\label{sec:4_2_2_Training_procedure}

\begin{figure}[t]
    \centering
    \includegraphics[width=\linewidth]{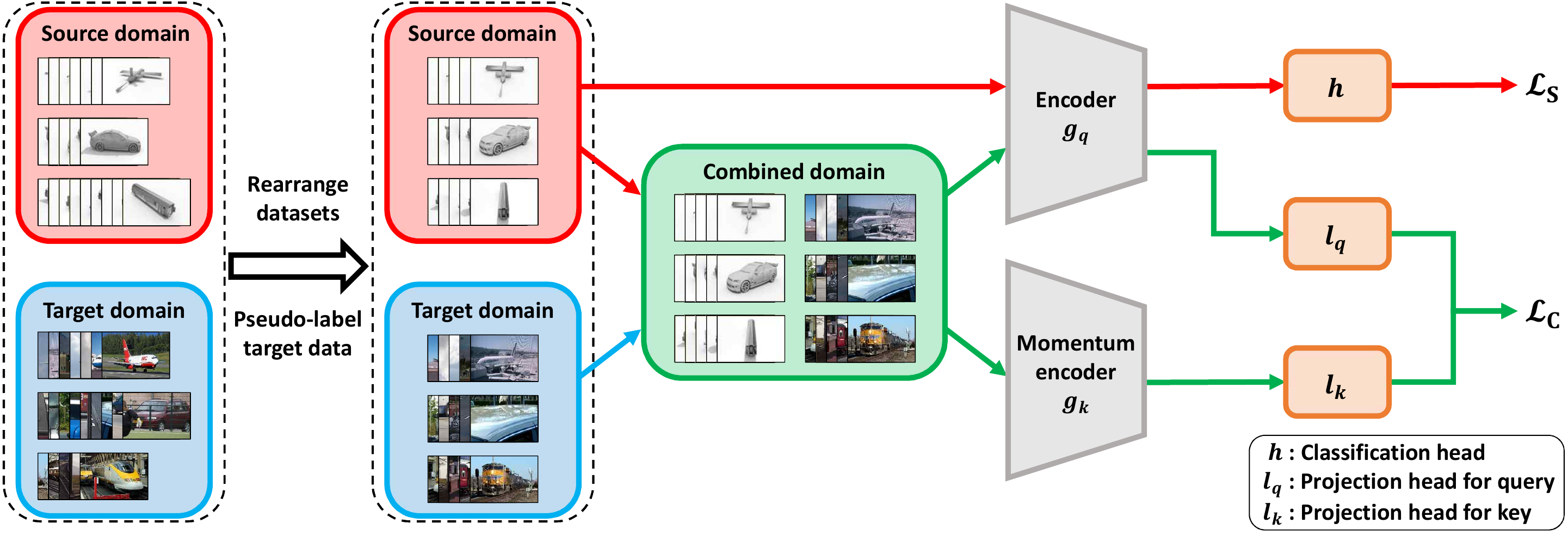}
    \caption{An overview of the JCL training process. To alleviate the problem of tradeoff between a small joint error and marginal distribution alignment when the label distributions are substantially different~\cite{pmlr-v97-zhao19a}, we propose rearranging the datasets to obtain uniform label distributions. For contrastive learning, we adopt \textit{Momentum Contrast}~\cite{he2019momentum}, which maintains a queue and a moving-averaged encoder on-the-fly to enable a large and consistent dictionary. After the encoders, a fully connected (FC) layer is applied for classification, and another FC layer is employed for contrastive loss.}
    \label{fig:jcl_overview}
\end{figure}

In this section, we formulate the loss functions and architecture of the method based on the aforementioned theoretical frameworks.
The overview of JCL is illustrated in Figure~\ref{fig:jcl_overview} and its pseudo-code is provided in Appendix D.
\textit{Momentum Contrast} (MoCo)~\cite{he2019momentum} is adopted as the proposed contrastive learning structure, with an encoder $g_q$ with parameters $\theta_q$ and a momentum-updated encoder $g_k$ with parameters $\theta_k$ for feature transformation $\mathcal{X} \mapsto \mathcal{Z}$.
$\theta_k$ are updated by $\theta_k \leftarrow m\theta_k + (1-m)\theta_q$, where $m \in [0, 1)$ is a momentum coefficient.
A fully connected (FC) layer projection head $l: \mathcal{Z} \mapsto \mathcal{W}$ is implemented to map the encoded representations to the space where the InfoNCE loss is applied.
Empirical tests determine that it is beneficial to define InfoNCE loss in the projected space $\mathcal{W}$ rather than $\mathcal{Z}$, which is in agreement with the results of~\citet{chen2020simple}.
For the feature pairs in the InfoNCE loss, an encoded query $\boldsymbol{w}_q=l_q(g_q(\boldsymbol{x}))$ and a key $\boldsymbol{w}_k=l_k(g_k(\boldsymbol{x}))$ from the queue of encoded features are used, where $l_q$ and $l_k$ are defined in a similar manner to $g_q$ and $g_k$, respectively.
We obtain the new keys on-the-fly by the momentum encoder and retain the queue of keys.
For the critic function $c$, we employ a cosine similarity function $\mathrm{sim}(\boldsymbol{u}, \boldsymbol{v}) = \boldsymbol{u}^\intercal \boldsymbol{v} / \|\boldsymbol{u}\| \|\boldsymbol{v}\|$ with a temperature hyper-parameter $\tau$ according to~\cite{wu2018unsupervised}.
Our InfoNCE loss $\mathcal{L}_\mathrm{c}$ is then formulated as follows:

\begin{equation}
    \mathcal{L}_{\mathrm{c}} = \E_{w_q \sim \mathcal{D}_U^W} \Bigg[ \E_{w_k^{+}} \Bigg[ -\log \frac{\exp (\mathrm{sim}(w_q, w_k^{+})/\tau)}{\sum_{w_k \in N_k \cup \{w_k^{+}\}}\exp (\mathrm{sim}(w_q, w_k) / \tau)} \Bigg] \Bigg],
\end{equation}

where $w_k^{+}$ is a feature that has the same label as $w_q$ and $N_k$ is a set of features that have different labels from $w_q$.
For the classification task, we have another FC layer $h$ as a classification head.
To guarantee a small source error, we employ the broadly used cross-entropy loss,

\begin{equation}
    \mathcal{L}_{\mathrm{s}} = \E_{(\boldsymbol{x}_s, y_s) \sim \mathcal{D}_S} \Big[ -\log h(g_q(\boldsymbol{x}_s))_{y_s} \Big].
\end{equation}

Combining $\mathcal{L}_\mathrm{s}$ and $\mathcal{L}_\mathrm{c}$ with a hyper-parameter $\gamma$, the overall objective is formulated as follows:

\begin{equation}
    \min_\theta \mathcal{L}_\mathrm{s} + \gamma \mathcal{L}_\mathrm{c}.
\end{equation}

The labels of the target data are required to recognize whether or not the two samples of the combined domain have the same label.
To facilitate this, we generate the pseudo-labels of the target data.
In particular, we perform spherical K-means clustering of target data on the feature space $\mathcal{Z}$ and assign labels at the begining of each epoch.
If the distance between a target sample and its assigned cluster center is larger than a constant $d$, then the target sample is excluded from the combined dataset.

\citet{pmlr-v97-zhao19a} showed that if the marginal label distributions of source and target domains are substantially different, a small joint error is not achievable while finding an invariant representation.
To address this problem, we suggest the reformulation of datasets to provide uniform label distributions, in which the number of data per class is equalized by data rearrangement.

\section{Experiments}
\label{sec:experiments}
\subsection{Datasets and baselines}
\textbf{ImageCLEF-DA}\footnote{\url{https://www.imageclef.org/2014/adaptation}}
is a real-world dataset consisting of three domains: \textit{Caltech-256} (\textbf{C}), \textit{ImageNet ILSVRC 2012} (\textbf{I}), and \textit{Pascal VOC 2012} (\textbf{P}).
Each domain contains 600 images from 12 common classes.
We evaluated all six possible transfer tasks among these three domains.

\textbf{VisDA-2017}~\cite{visda2017}
is a dataset for the synthetic-to-real transfer task and has a high dataset shift.
It includes 152,397 synthetic 2D renderings of 3D models and 55,388 real images across 12 classes.

\textbf{Baselines.}
We compare JCL with marginal distribution matching methods: Deep Adaptation Network (\textbf{DAN})~\cite{pmlr-v37-long15}, Domain Adversarial Neural Network (\textbf{DANN})~\cite{JMLR:v17:15-239}, and Joint Adaptation Network (\textbf{JAN})~\cite{long2017deep} and
also with methods that endeavor to learn discriminative features: Multi-Adversarial Domain Adaptation (\textbf{MADA})~\cite{pei2018multi}, Conditional Domain Adversarial Network (\textbf{CDAN})~\cite{long2018conditional}, Adversarial Dropout Regularization (\textbf{ADR})~\cite{saito2018adversarial}, Maximum Classifier Discrepancy (\textbf{MCD})~\cite{saito2018maximum}, Cluster Alignment with a Teacher (\textbf{CAT})~\cite{deng2019cluster}, and Contrastive Adaptation Network (\textbf{CAN})~\cite{kang2019contrastive}. 

\subsection{Implementation details}
We follow the standard experimental protocols for unsupervised domain adaptation~\cite{pmlr-v37-ganin15, long2017deep} and report the average accuracy over three independent runs.
To select the hyper-parameters, we use the same protocol as the one described in~\cite{pmlr-v37-long15}:
we train a source classifier and a domain classifier on a validation set that consists of labeled source data and unlabeled target data, and then, we jointly evaluate the test errors of the classifiers.
The selected hyper-parameters are detailed in Appendix E.

We adopt \textbf{ResNet-50} and \textbf{ResNet-101}~\cite{he2016deep} as base networks for ImageCLEF-DA and VisDA-2017, respectively, and we employ domain-specific batch normalization layers.
We finetune from ImageNet~\cite{deng2009imagenet} pre-trained models, with the exception of the last FC layer, which we replace with the task-specific FC layer.
We also add another FC layer with an output dimension of 256 for contrastive learning.
We utilize mini-batch SGD with momentum of 0.9 and follow the same learning rate schedule as~\cite{pmlr-v37-long15, JMLR:v17:15-239, long2017deep}: the learning rate $\eta_p$ is adjusted according to $\eta_p = \eta_0 (1+\alpha p)^{-\beta}$, where $p$ is the training progress that increases from 0 to 1.
The $\eta_0$ is the initial learning rate, which is set to 0.001 for the pre-trained layers and 0.01 for the added FC layers.
The $\alpha$ and $\beta$ are fixed to 10 and 0.75, respectively.
More details about the implementation is provided in Appendix F.

\subsection{Results}
\begin{table}
  \small
  \caption{Accuracy (\%) on ImageCLEF-DA for unsupervised domain adaptation (ResNet-50).}
  \label{imageclef_result}
  \centering
  \resizebox{\textwidth}{!}{\begin{tabular}{lccccccc}
    \toprule
    Method 
    & I $\rightarrow$ P
    & P $\rightarrow$ I
    & I $\rightarrow$ C
    & C $\rightarrow$ I
    & C $\rightarrow$ P
    & P $\rightarrow$ C
    & Average 
    \\
    \midrule
    ResNet-50~\cite{he2016deep}
    & 74.8 $\pm$ 0.3 & 83.9 $\pm$ 0.1 & 91.5 $\pm$ 0.3 & 78.0 $\pm$ 0.2 & 65.5 $\pm$ 0.3 & 91.2 $\pm$ 0.3 & 80.7 \\
    DANN~\cite{JMLR:v17:15-239}     
    & 75.0 $\pm$ 0.6 & 86.0 $\pm$ 0.3 & 96.2 $\pm$ 0.4 & 87.0 $\pm$ 0.5 & 74.3 $\pm$ 0.5 & 91.5 $\pm$ 0.6 & 85.0 \\
    DAN~\cite{pmlr-v37-long15}          
    & 74.5 $\pm$ 0.4 & 82.2 $\pm$ 0.2 & 92.8 $\pm$ 0.2 & 86.3 $\pm$ 0.4 & 69.2 $\pm$ 0.4 & 89.8 $\pm$ 0.4 & 82.5 \\  
    JAN~\cite{long2017deep}             
    & 76.8 $\pm$ 0.4 & 88.0 $\pm$ 0.2 & 94.7 $\pm$ 0.2 & 89.5 $\pm$ 0.3 & 74.2 $\pm$ 0.3 & 91.7 $\pm$ 0.3 & 85.8 \\
    MADA~\cite{pei2018multi}             
    & 75.0 $\pm$ 0.3 & 87.9 $\pm$ 0.2 & 96.0 $\pm$ 0.3 & 88.8 $\pm$ 0.3 & 75.2 $\pm$ 0.2 & 92.2 $\pm$ 0.3 & 85.8 \\
    CDAN+E~\cite{long2018conditional}             
    & 77.7 $\pm$ 0.3 & 90.7 $\pm$ 0.2 & \textbf{97.7} $\pm$ 0.3 & 91.3 $\pm$ 0.3 & 74.2 $\pm$ 0.2 & 94.3 $\pm$ 0.3 & 87.7 \\
    CAT~\cite{deng2019cluster}         
    & 77.2 $\pm$ 0.2 & 91.0 $\pm$ 0.3 & 95.5 $\pm$ 0.3 & 91.3 $\pm$ 0.3 & 75.3 $\pm$ 0.6 & 93.6 $\pm$ 0.5 & 87.3 \\
    \midrule
    JCL 
    & \textbf{78.1} $\pm$ 0.3 & \textbf{93.4} $\pm$ 0.2 & \textbf{97.7} $\pm$ 0.2 & \textbf{93.5} $\pm$ 0.4 & \textbf{78.1} $\pm$ 0.7 & \textbf{97.7} $\pm$ 0.4 & \textbf{89.8} \\
    \bottomrule
  \end{tabular}}
\end{table}

The results obtained using the ImageCLEF-DA dataset are reported in Table~\ref{imageclef_result}.
For all six adaptation scenarios, our proposed method outperforms the other baseline methods and achieves state-of-the-art accuracy.
In particular, the proposed method surpasses CAT by a substantial margin, validating the effectiveness of jointly learning discriminative features and the discussed information-theoretic guarantees.
Moreover, the methods that consider conditional distributions achieve higher accuracies than those that focus on marginal distribution matching.
These results suggest that learning discriminative features to minimize the joint hypothesis error is more crucial than general alignment.

\begin{figure}
\parbox[t]{0.45\textwidth}{\null
\centering
\small
  \vskip-\abovecaptionskip
  \captionof{table}[t]{Accuracy (\%) on VisDA-2017 for unsupervised domain adaptation (ResNet-101).}%
  \vskip\abovecaptionskip
  \label{visda_result}
  \begin{tabular}{lccc}
    \toprule
    Method 
    & car  
    & truck
    & Average 
    \\
    \midrule
    ResNet-101~\cite{he2016deep} & \textbf{91.7} & 25.9 & 49.4 \\
    DANN~\cite{JMLR:v17:15-239} & 44.3 & 7.8 & 57.4 \\
    DAN~\cite{pmlr-v37-long15} & 87.0 & 42.2 & 62.8 \\
    JAN~\cite{long2017deep} & 86.3 & 54.5 & 65.7 \\
    MCD~\cite{saito2018maximum} & 64.0 & 25.8 & 71.9 \\
    ADR~\cite{saito2018adversarial} & 65.3 & 32.3 & 74.8 \\
    SE~\cite{french2018selfensembling} & 58.6 & 47.9 & 84.3 \\
    CAN~\cite{kang2019contrastive} & 74.3 & 59.9 & 87.2 \\
    \midrule
    JCL & 66.8 & \textbf{71.8} & \textbf{87.6} \\
    \bottomrule
  \end{tabular}
}
\hfill
\parbox[t]{0.55\textwidth}{\null
  \begin{subfigure}[b]{0.27\textwidth}
    \centering
    \includegraphics[width=\textwidth]{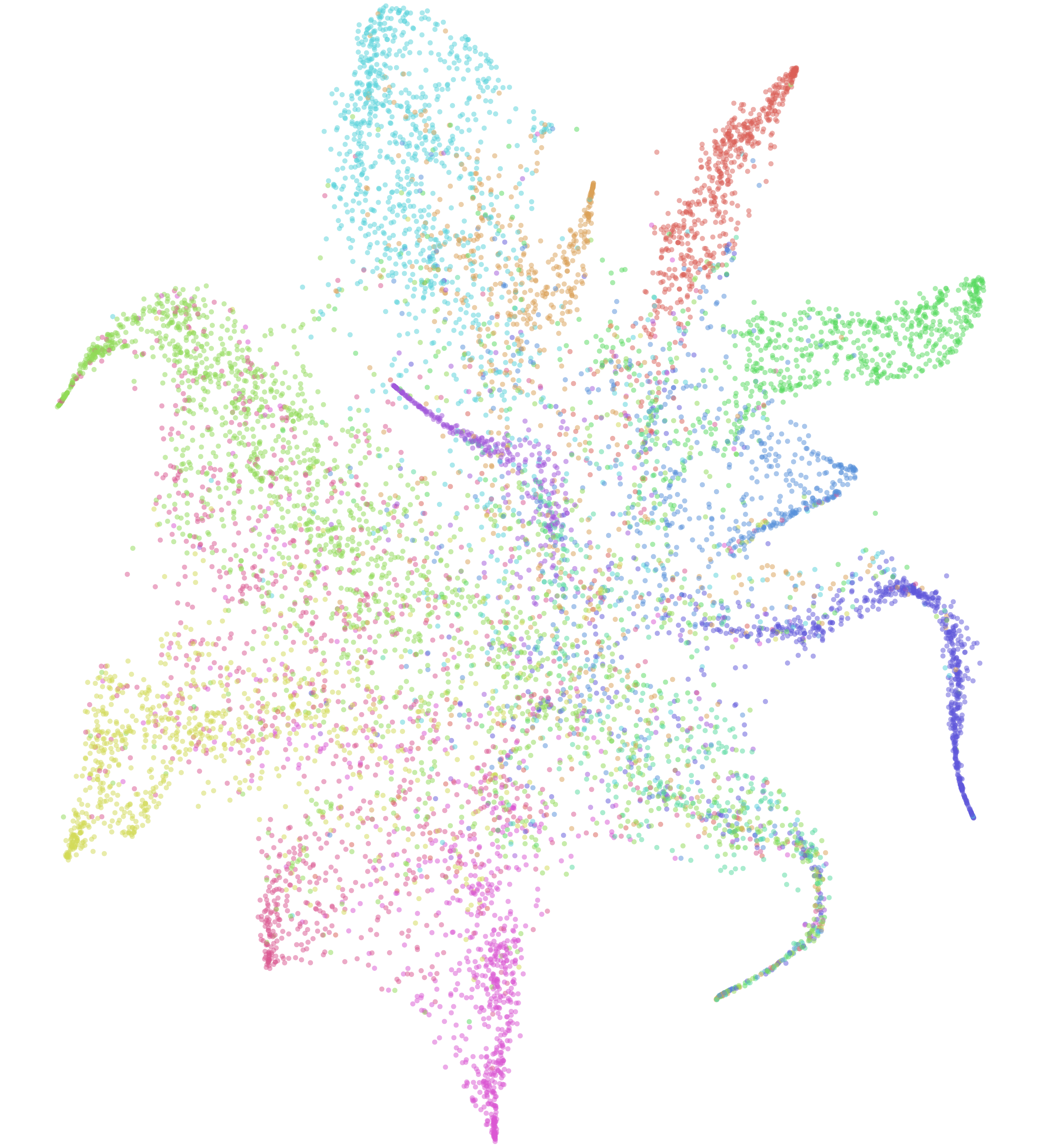}
  \end{subfigure}
  \hfill
  \begin{subfigure}[b]{0.27\textwidth}
    \centering
    \includegraphics[width=\textwidth]{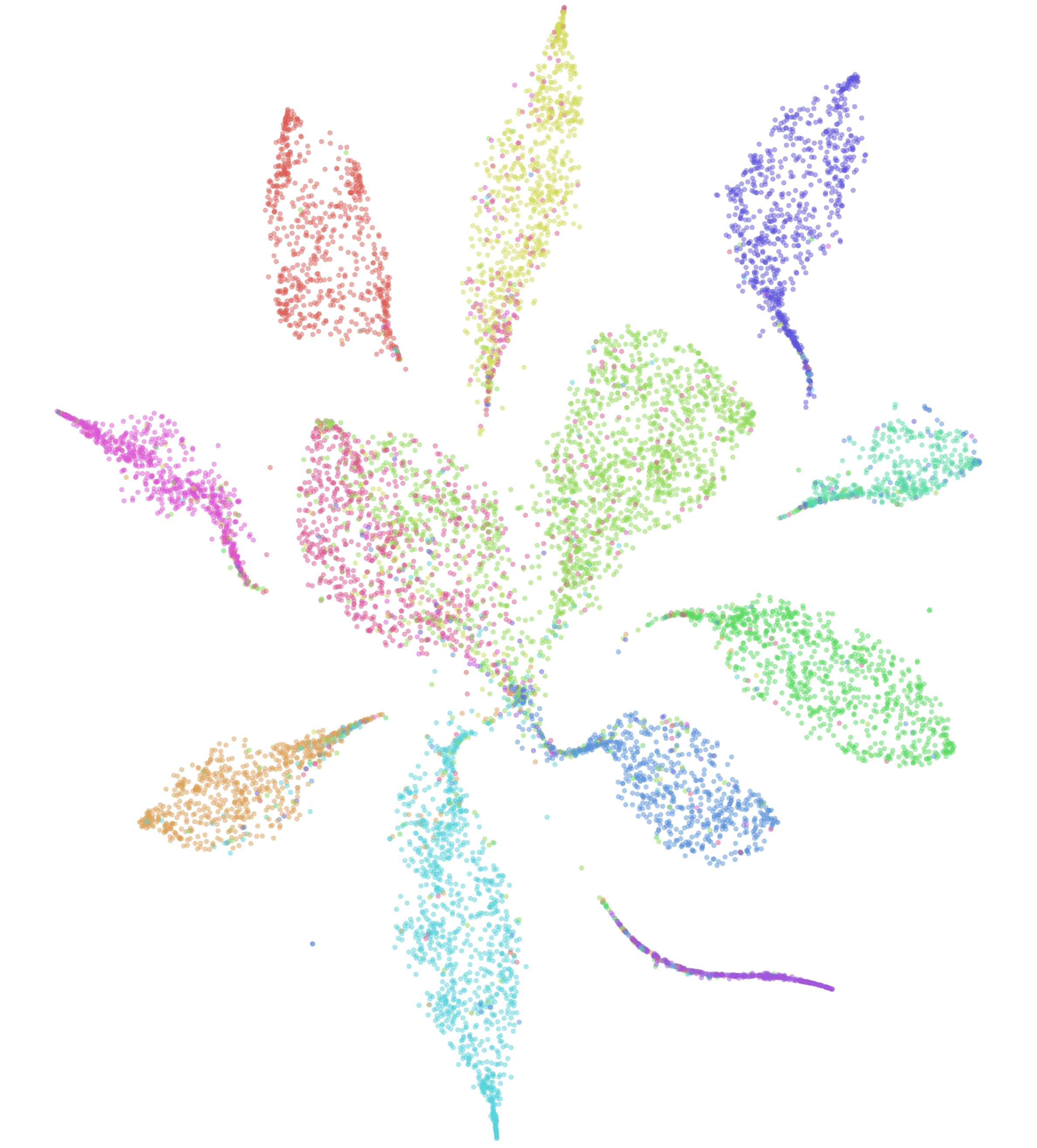}
  \end{subfigure}
  \captionof{figure}{Visualization for different methods (best viewed in color). \textbf{Left}: t-SNE of DANN. \textbf{Right}: JCL.}%
  \label{fig:tsne}
}
\end{figure}

In Table~\ref{visda_result}, the accuracies obtained for two particular classes and the average accuracy over all twelve classes on the VisDA-2017 transfer task are reported.
All the results pertaining to the VisDA-2017 transfer task are provided in Appendix G owing to space constraints in the main manuscript.
The selected classes are "\textit{car}," wherein the proposed method achieves the lowest accuracy among the twelve objects and "\textit{truck}," which is the most challenging object among the twelve objects.
Notably, the proposed method boosts the accuracy of the \textit{truck} class by a significant margin, and, on average, it outperforms the other baseline methods.
In particular, it advances the lowest accuracy among the twelve objects of CAN (59.9\%) by 6.9\%.
These results can be attributed to the MI maximization between a feature and its label which trades-off maximizing entropy $H(\boldsymbol{z})$ and minimizing conditional entropy $H(\boldsymbol{z}|y)$, and thus avoids degenerate solutions~\cite{caron2018deep} (see Appendix B for more details).

We visualize the learned target representations of the VisDA-2017 task by t-SNE~\cite{maaten2008visualizing} in Figure~\ref{fig:tsne} to compare our method with DANN in terms of feature discriminability.
While aligning the marginal distributions of the source and target domains, the features are not well discriminated with DANN.
On the contrary, the target features learned using our method are clearly discriminated, demonstrating that our objective to maximize the JS divergence between conditional distributions is achieved.


\subsection{Ablation studies}
\label{sec:ablation_studies}

\begin{figure}
     \centering
     \begin{subfigure}[b]{0.32\textwidth}
         \centering
         \includegraphics[width=\textwidth]{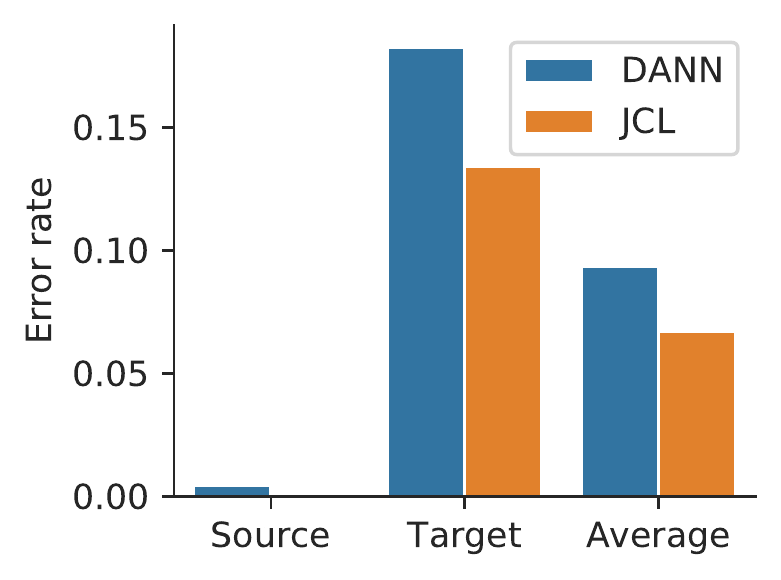}
         \caption{}
         \label{fig:error_rate}
     \end{subfigure}
     \hfill
     \begin{subfigure}[b]{0.32\textwidth}
         \centering
         \includegraphics[width=\textwidth]{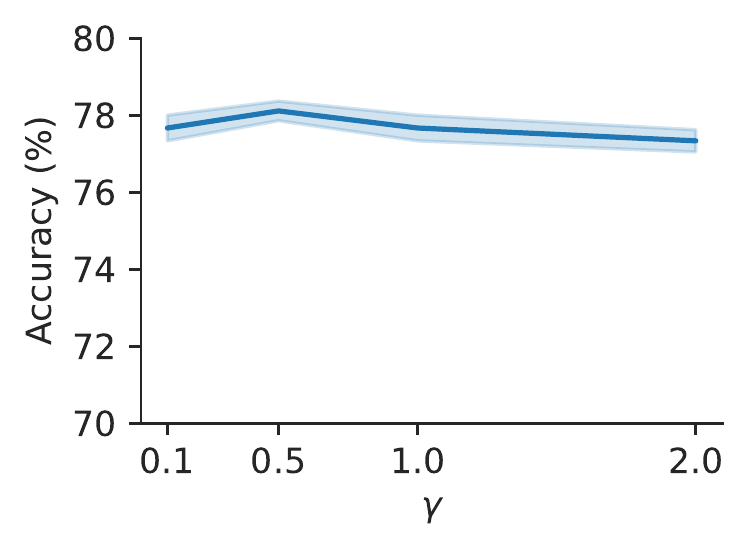}
         \caption{}
         \label{fig:sensitivity_i_p}
     \end{subfigure}
     \hfill
     \begin{subfigure}[b]{0.32\textwidth}
         \centering
         \includegraphics[width=\textwidth]{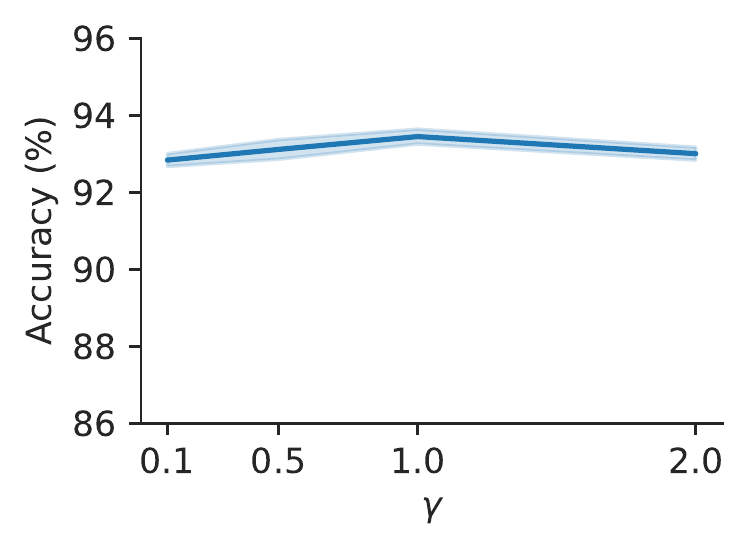}
         \caption{}
         \label{fig:sensitivity_p_i}
     \end{subfigure}
        \caption{(a) Classification error rate on the learned representations. (b-c) The accuracy sensitivity of JCL to $\gamma$ on \textbf{I} $\rightarrow$ \textbf{P} (b) and \textbf{P} $\rightarrow$ \textbf{I} (c). The results for other tasks are similar.}
        \label{fig:ablation_studies}
\end{figure}

To investigate the effectiveness of our method in minimizing the joint hypothesis error by learning discriminative representations, we conduct the same pilot analysis as~\citet{chen2019transferability};
we train a linear classifier on the representations learned using DANN and our method.
The linear classifier is trained on both source and target data using the labels.
The average error rate of the linear classifier corresponds to half of the ideal joint hypothesis error.
The results are shown in Figure~\ref{fig:error_rate}.
We can observe that the ideal joint hypothesis error of the representation learned using our method is significantly lower than that learned using DANN.
This implies that the proposed method is effective in achieving our objective to enhance feature discriminability.

We investigate the sensitivity of JCL to the weight hyper-parameter $\gamma$, and the results are shown in Figure~\ref{fig:sensitivity_i_p} and~\ref{fig:sensitivity_p_i}.
We could observe that JCL is not sensitive to the change in the value of $\gamma$.


\section{Conclusion}
\label{sec:conclusion}
In this study, we suggest an alternative upper bound on the target error to explicitly manage the joint hypothesis error.
The proposed upper bound with the joint hypothesis error provides a new perspective on the target error that the joint optimization on the both domains is demanded.
Further, a novel approach to domain adaptation, JCL, is proposed to minimize the joint error.
The proposed approach differs from previous domain adaptation methods that consider conditional distributions, as it can maximize the JS divergence between class-conditional distributions with information-theoretic guarantees.
The effectiveness of the proposed method is validated with several experiments.

\section{Broader impact}
\label{sec:broader_impact}
Our study advances unsupervised domain adaptation, which utilizes a labeled source domain to classify the target domain well. 
This study is important for our future society because it circumvents a dataset shift.
First, individuals from different regions can derive benefits in utilizing deep neural network models by the lowered barrier between regions or cultures. 
Second, researchers and developers can save their efforts in distributing models worldwide by minimizing costs. 
We believe that the assumptions, data, and algorithms adopted in this study do not have any ethical issues.


\medskip

\begin{small}
\bibliographystyle{abbrvnat}
\bibliography{references}
\end{small}

\appendix
\def\thesection{\Alph{section}}

\section{Proofs}
In this section, we provide the proofs of Theorem~\ref{alternative_theorem} and Theorem~\ref{pseudo_alternative_theorem} in the main manuscript.
We first introduce lemmas that are useful in proving the theorems.

\begin{lemma}
\label{triangle_inequality}
Let $\mathcal{H}$ be a hypothesis space and $\mathcal{D}$ be any distribution over input space $\mathcal{X}$. Then $\forall h, h^\prime, h^{\prime\prime} \in \mathcal{H}$, the following triangle inequality holds:
\[\epsilon_\mathcal{D} (h, h^\prime) \leq \epsilon_\mathcal{D} (h, h^{\prime\prime}) + \epsilon_\mathcal{D} (h^{\prime\prime}, h^\prime).\]
\end{lemma}

\begin{proof}
From the definition of the error and the triangle inequality of norm, we have
\[
\begin{split}
\epsilon_\mathcal{D} (h, h^\prime) 
&= \mathbb{E}_{\boldsymbol{x} \sim \mathcal{D}}[\abs{h(\boldsymbol{x}) - h^\prime(\boldsymbol{x})}] \\
&= \mathbb{E}_{\boldsymbol{x} \sim \mathcal{D}}[\abs{h(\boldsymbol{x}) - h^\prime(\boldsymbol{x}) + h^{\prime\prime}(\boldsymbol{x}) - h^{\prime\prime}(\boldsymbol{x})}] \\
&\leq \mathbb{E}_{\boldsymbol{x} \sim \mathcal{D}}[\abs{h(\boldsymbol{x}) - h^{\prime\prime}(\boldsymbol{x})} + \abs{h^{\prime\prime}(\boldsymbol{x}) - h^\prime(\boldsymbol{x})}] \\
&= \mathbb{E}_{\boldsymbol{x} \sim \mathcal{D}}[\abs{h(\boldsymbol{x}) - h^{\prime\prime}(\boldsymbol{x})}] + \mathbb{E}_{\boldsymbol{x} \sim \mathcal{D}}[\abs{h^{\prime\prime}(\boldsymbol{x}) - h^\prime(\boldsymbol{x})}] \\
&= \epsilon_\mathcal{D} (h, h^{\prime\prime}) + \epsilon_\mathcal{D} (h^{\prime\prime}, h^\prime).
\end{split}
\]
\end{proof}

\begin{lemma}[\cite{ben2010theory}]
\label{h_divergence_inequality}
For any hypothesis $h, h^\prime \in \mathcal{H}$, the following inequality holds:
\[
\abs{\epsilon_S (h, h^\prime) - \epsilon_T (h, h^\prime)} \leq \frac{1}{2}d_{\mathcal{H}\Delta\mathcal{H}}(\mathcal{D}_S, \mathcal{D}_T).
\]
\end{lemma}

\begin{proof}
From the definition of the $\mathcal{H}\Delta\mathcal{H}$-distance, we have
\[
\begin{split}
d_{\mathcal{H}\Delta\mathcal{H}}(\mathcal{D}_S, \mathcal{D}_T)
&= 2 \sup_{h, h^\prime \in \mathcal{H}} \abs{\mathrm{Pr}_{\boldsymbol{x} \sim \mathcal{D}_S} [h(\boldsymbol{x}) \neq h^\prime(\boldsymbol{x})] - \mathrm{Pr}_{\boldsymbol{x} \sim \mathcal{D}_T} [h(\boldsymbol{x}) \neq h^\prime(\boldsymbol{x})]} \\
&= 2 \sup_{h, h^\prime \in \mathcal{H}} \abs{\epsilon_S (h, h^\prime) - \epsilon_T (h, h^\prime)} \\
&\geq 2 \abs{\epsilon_S (h, h^\prime) - \epsilon_T (h, h^\prime)}.
\end{split}
\]
\end{proof}

\alternative*

\begin{proof}
From Lemma~\ref{triangle_inequality}, Lemma~\ref{h_divergence_inequality}, and the definition of $\mathcal{D}_U$ and $f_U$, we have
\[
\begin{split}
\epsilon_{T}(h, f_T)
&\leq \epsilon_{T}(h, f_U) + \epsilon_{T}(f_U, f_T) \\
&= \epsilon_{S}(h, f_U) + \epsilon_{T}(h, f_U) + \epsilon_{T}(f_U, f_T) - \epsilon_{S}(h, f_U) \\
&\leq \epsilon_{S}(h, f_U) + \epsilon_{T}(h, f_U) + \epsilon_{T}(f_U, f_T) + \epsilon_{S}(h, f_S) - \epsilon_{S}(f_U, f_S) \\
&= \epsilon_{S}(h, f_S) + \epsilon_{S}(h, f_U) + \epsilon_{T}(h, f_U) \\
&\phantom{--------}+ \mathbb{E}_{\boldsymbol{x} \sim \mathcal{D}_T}[\abs{f_U(\boldsymbol{x}) - f_T(\boldsymbol{x})}] - \mathbb{E}_{\boldsymbol{x} \sim \mathcal{D}_S}[\abs{f_U(\boldsymbol{x}) - f_S(\boldsymbol{x})}] \\
&= \epsilon_{S}(h, f_S) + \epsilon_{S}(h, f_U) + \epsilon_{T}(h, f_U) \\
&\phantom{--------}+ \frac{1}{2}\mathbb{E}_{\boldsymbol{x} \sim \mathcal{D}_T}[\abs{f_S(\boldsymbol{x}) - f_T(\boldsymbol{x})}] - \frac{1}{2}\mathbb{E}_{\boldsymbol{x} \sim \mathcal{D}_S}[\abs{f_S(\boldsymbol{x}) - f_T(\boldsymbol{x})}] \\
&= \epsilon_{S}(h, f_S) + \epsilon_{S}(h, f_U) + \epsilon_{T}(h, f_U) + \frac{1}{2}\epsilon_{T}(f_S, f_T) - \frac{1}{2}\epsilon_{S}(f_S, f_T) \\
&\leq \epsilon_{S}(h, f_S) + \epsilon_{S}(h, f_U) + \epsilon_{T}(h, f_U) + \frac{1}{2}\abs{\epsilon_{T}(f_S, f_T) - \epsilon_{S}(f_S, f_T)} \\
&\leq \epsilon_{S}(h, f_S) + \epsilon_{S}(h, f_U) + \epsilon_{T}(h, f_U) + 
\frac{1}{4}d_{\mathcal{H}\Delta\mathcal{H}}(\mathcal{D}_S, \mathcal{D}_T) \\
&= \epsilon_{S}(h, f_S) + \frac{1}{4}d_{\mathcal{H}\Delta\mathcal{H}}(\mathcal{D}_S, \mathcal{D}_T) \\
&\phantom{--------}+ \int \phi_S(\boldsymbol{x}) \abs{h(\boldsymbol{x}) - f_U(\boldsymbol{x})} \mathrm{d}\boldsymbol{x} + \int \phi_T(\boldsymbol{x}) \abs{h(\boldsymbol{x}) - f_U(\boldsymbol{x})} \mathrm{d}\boldsymbol{x} \\
&= \epsilon_{S}(h, f_S) + \frac{1}{4}d_{\mathcal{H}\Delta\mathcal{H}}(\mathcal{D}_S, \mathcal{D}_T) + 2 \int \frac{1}{2}(\phi_S(\boldsymbol{x}) + \phi_T(\boldsymbol{x})) \abs{h(\boldsymbol{x}) - f_U(\boldsymbol{x})} \mathrm{d}\boldsymbol{x} \\
&= \epsilon_{S}(h, f_S) + \frac{1}{4}d_{\mathcal{H}\Delta\mathcal{H}}(\mathcal{D}_S, \mathcal{D}_T) + 2 \int \phi_U(\boldsymbol{x}) \abs{h(\boldsymbol{x}) - f_U(\boldsymbol{x})} \mathrm{d}\boldsymbol{x} \\
&= \epsilon_{S}(h, f_S) + \frac{1}{4}d_{\mathcal{H}\Delta\mathcal{H}}(\mathcal{D}_S, \mathcal{D}_T) + 2 \epsilon_{U}(h, f_U).
\end{split}
\]
\end{proof}

\pseudo*

\begin{proof}
By Lemma~\ref{triangle_inequality}, the following inequality holds.
\[
\epsilon_{T}(h, f_T) \leq \epsilon_{T}(h, f_{\hat{T}}) + \epsilon_{T}(f_T, f_{\hat{T}}).
\]
Meanwhile, using the same process in the proof of Theorem~\ref{alternative_theorem}, we know that
\[\epsilon_{T}(h, f_{\hat{T}}) \leq \epsilon_{S}(h, f_S) + \frac{1}{4}d_{\mathcal{H}\Delta\mathcal{H}}(\mathcal{D}_S, \mathcal{D}_T) + 2 \epsilon_{U}(h, f_{\hat{U}}).\]
Combining the above two inequalities yields the proof.
\end{proof}

\section{Generalization to multiclass classification}

Here, we introduce how the proposed theoretical background can be generalized to a multiclass classification problem and explain why degenerate solutions can be avoided from an information-theoretic perspective.
The generalized Jensen-Shannon (JS) divergence is defined as:

\begin{definition} [\cite{lin1991divergence}]
Let $\mathcal{D}_1, \mathcal{D}_2, \cdots, \mathcal{D}_n$ be n probability distributions with weights $\pi_1, \pi_2, \cdots, \pi_n$, respectively, and let $Z_1, Z_2, \cdots, Z_n$ be random variables with distributions $\mathcal{D}_1, \mathcal{D}_2, \cdots, \mathcal{D}_n$, respectively.
Then, the generalized JS divergence is defined as:

\[
D_{\mathrm{JS}}^\pi (\mathcal{D}_1, \mathcal{D}_2, \cdots, \mathcal{D}_n) = H (Z) - \sum_{i=1}^n \pi_i H(Z_i),
\]

where $\pi$ is $(\pi_1, \pi_2, \cdots, \pi_n)$ and $Z$ is a random variable with the mixture distribution of $\mathcal{D}_1, \mathcal{D}_2, \cdots, \mathcal{D}_n$ with weights $\pi_1, \pi_2, \cdots, \pi_n$, respectively.
\end{definition}

The generalized JS divergence measures the overall difference among a finite number of probability distributions.
Notably, for a fixed $\pi$, the Bayes probability of error~\cite{hellman1970probability} is minimized if the generalized JS divergence is maximized~\cite{lin1991divergence}.
With this divergence, we can generalize our objective to learn discriminative features in the representation space, $\mathcal{Z}$, as follows.

\begin{equation}
    \max_{\theta_g} D_{\mathrm{JS}}^\pi(\mathcal{D}^{\mathcal{Z}}_{U|0}, \mathcal{D}^{\mathcal{Z}}_{U|1}, \cdots, \mathcal{D}^{\mathcal{Z}}_{U|C-1}),
\end{equation}

where $\pi$ denotes the marginal label distribution and $C$ denotes the number of classes.
From the definition of the generalized JS divergence, we know that

\begin{equation}
\label{general_js_and_mi}
    \begin{split}
        D_{\mathrm{JS}}^\pi(\mathcal{D}^{\mathcal{Z}}_{U|0}, \mathcal{D}^{\mathcal{Z}}_{U|1}, \cdots, \mathcal{D}^{\mathcal{Z}}_{U|C-1})
        &= H (\boldsymbol{Z}_{U|Y}) - \sum_{y=1}^n \pi_y H(\boldsymbol{Z}_{U|y}) \\
        &= H (\boldsymbol{Z}_{U|Y}) - \sum_{y=1}^n P(Y=y) H(\boldsymbol{Z}_{U|Y} | Y=y) \\
        &= H (\boldsymbol{Z}_{U|Y}) - H (\boldsymbol{Z}_{U|Y} | Y) \\
        &= I (Y;\boldsymbol{Z}_{U|Y}).
    \end{split}
\end{equation}

Therefore, we can transform our objective as follows:

\begin{equation}
    \max_{\theta_g} I(Y;\boldsymbol{Z}_{U|Y}).
\end{equation}

With the same theoretical framework introduced in the main manuscript, we can optimize this objective with the InfoNCE loss.

\paragraph{Avoiding degenerate solutions}
Algorithms that learn discriminative representations by alternating pseudo-labeling and updating the parameters of the network are susceptible to trivial solutions~\cite{caron2018deep}, referred to as degenerate solutions.
For example, if the majority of samples are assigned to a few clusters, it is easy to discriminate between features, but this is unfavorable for downstream tasks.
The proposed approach can avoid the tendency towards degenerate solutions since the method maximizes the mutual information (MI) between a feature and its label.
From Equation~\ref{general_js_and_mi}, we can observe that maximizing the MI, $I (Y;\boldsymbol{Z}_{U|Y})$, trades off maximizing the entropy, $H (\boldsymbol{Z}_{U|Y})$, and minimizing the conditional entropy, $H (\boldsymbol{Z}_{U|Y} | Y)$.
Only minimizing the conditional entropy can be vulnerable to the degenerate solutions, but the objective also includes the entropy maximization, which cannot be achieved in the degenerate solutions.
Therefore, the objective naturally balances discriminative representation learning with dispersed features and avoids the degenerate solutions.

\section{More details about the data processing inequality}
Here, we provide the formal derivation of $I(\boldsymbol{Z}_{U|Y}^{(1)}; \boldsymbol{Z}_{U|Y}^{(2)}) \leq I(Y;\boldsymbol{Z}_{U|Y}^{(1)}, \boldsymbol{Z}_{U|Y}^{(2)})$, presented in the main manuscript, using the \textit{data processing inequality}.

For a given $Y$, we sample two different data, $\boldsymbol{X}_{U|Y}^{(1)}$ and $\boldsymbol{X}_{U|Y}^{(2)}$, from the same conditional distribution, $\mathcal{D}^{\mathcal{X}}_{U|Y}$.
Then, we obtain $\boldsymbol{Z}_{U|Y}^{(1)}$ and $\boldsymbol{Z}_{U|Y}^{(2)}$ from $\boldsymbol{X}_{U|Y}^{(1)}$ and $\boldsymbol{X}_{U|Y}^{(2)}$, respectively, through the feature transformation, $g$.
Therefore, $Y$, $\boldsymbol{X}_{U|Y}^{(1)}$, $\boldsymbol{X}_{U|Y}^{(2)}$, $\boldsymbol{Z}_{U|Y}^{(1)}$, and $\boldsymbol{Z}_{U|Y}^{(2)}$ satisfy the Markov relation:

\begin{equation}
    \boldsymbol{Z}_{U|Y}^{(1)} \leftarrow \boldsymbol{X}_{U|Y}^{(1)} \leftarrow Y \rightarrow \boldsymbol{X}_{U|Y}^{(2)} \rightarrow \boldsymbol{Z}_{U|Y}^{(2)},    
\end{equation}

and this is Markov equivalent to

\begin{equation}
    \boldsymbol{Z}_{U|Y}^{(1)} \rightarrow \boldsymbol{X}_{U|Y}^{(1)} \rightarrow Y \rightarrow \boldsymbol{X}_{U|Y}^{(2)} \rightarrow \boldsymbol{Z}_{U|Y}^{(2)}.    
\end{equation}

By the data processing inequality, we know that

\begin{equation}
    I(\boldsymbol{Z}_{U|Y}^{(1)}; \boldsymbol{Z}_{U|Y}^{(2)}) \leq I(Y; \boldsymbol{Z}_{U|Y}^{(1)}).
\end{equation}

Meanwhile, we can observe that the following Markov relation holds.

\begin{equation}
    Y \rightarrow (\boldsymbol{X}_{U|Y}^{(1)}, \boldsymbol{X}_{U|Y}^{(2)}) \rightarrow (\boldsymbol{Z}_{U|Y}^{(1)}, \boldsymbol{Z}_{U|Y}^{(2)}) \rightarrow \boldsymbol{Z}_{U|Y}^{(1)}.    
\end{equation}

Therefore, by the data processing inequality, we have

\begin{equation}
    I(Y; \boldsymbol{Z}_{U|Y}^{(1)}) \leq I(Y; (\boldsymbol{Z}_{U|Y}^{(1)}, \boldsymbol{Z}_{U|Y}^{(2)})).
\end{equation}

Combining the above two inequalities yields the desired derivation.

\section{Pseudo-code}
\begin{algorithm}[H]
\SetAlgoLined
\SetKwInOut{Input}{Input}
\SetKwInOut{Output}{Output}
\Input{Labeled source data from $\mathcal{D}_S$ and unlabeled target data from $\mathcal{D}_T$.}
 Initialize encoders $g_q$ and $g_k$; classification head $h$; projection heads $l_q$ and $l_k$; and a queue of K keys\;
 \While{iteration < max\_iteration}{
  Cluster the target data using spherical K-means\;
  Split them into a certain dataset with pseudo-labels and an uncertain dataset\;
  Rearrange the source and certain target datasets to obtain uniform label distributions\;
  \For{$i \leftarrow 1$ \KwTo iterations\_per\_epoch}{
    Sample mini-batches of the source data $(\boldsymbol{x}_s, y_s)$, certain target data $(\boldsymbol{x}_{tc}, \hat{y}_{tc})$, and uncertain target data $(\boldsymbol{x}_{tu})$\;
    $\boldsymbol{x}_s^q = \mathrm{pre}\text{-}\mathrm{process}(\boldsymbol{x}_s)$,
    $\boldsymbol{x}_s^k = \mathrm{pre}\text{-}\mathrm{process}(\boldsymbol{x}_s)$\;
    $\boldsymbol{x}_{tc}^q = \mathrm{pre}\text{-}\mathrm{process}(\boldsymbol{x}_{tc})$,
    $\boldsymbol{x}_{tc}^k = \mathrm{pre}\text{-}\mathrm{process}(\boldsymbol{x}_{tc})$,
    $\boldsymbol{x}_{tu} = \mathrm{pre}\text{-}\mathrm{process}(\boldsymbol{x}_{tu})$\;
    $\boldsymbol{z}_s^q = g_q(\boldsymbol{x}_s^q)$, $\boldsymbol{z}_s^k = g_k(\boldsymbol{x}_s^k)$\;
    Compute $\mathcal{L}_\mathrm{s}$ on $(h(\boldsymbol{z}_s^q), y_s)$ using Equation (5)\;
    $\boldsymbol{w}_s^q = l_q(\boldsymbol{z}_s^q)$,
    $\boldsymbol{w}_s^k = l_k(\boldsymbol{z}_s^k)$\;
    $\boldsymbol{w}_{tc}^q = l_q(g_q(\boldsymbol{x}_{tc}^q))$,
    $\boldsymbol{w}_{tc}^k = l_k(g_k(\boldsymbol{x}_{tc}^k))$\;
    Forward the uncertain target data to train the batch normalization layers, $g_q(\boldsymbol{x}_{tu})$\;
    Merge $\boldsymbol{w}_s^q$ and $\boldsymbol{w}_{tc}^q$ to obtain $\boldsymbol{w}_u^q$,
    and merge $\boldsymbol{w}_s^k$ and $\boldsymbol{w}_{tc}^k$ to obtain $\boldsymbol{w}_u^k$\;
    $\mathrm{enqueue}(\mathrm{queue}, \boldsymbol{w}_u^k)$\;
    $\mathrm{dequeue}(\mathrm{queue})$\;
    Compute $\mathcal{L}_\mathrm{c}$ on $(\boldsymbol{w}_u^q, \mathrm{queue})$ using Equation (4)\;
    Update the query network parameters, $\theta_q$\, with SGD\;
    Momentum update the key network parameters, $\theta_k$\;
  }
 }
 \caption{Training procedure for JCL}
\end{algorithm}

\section{Selected hyper-parameters}
We tuned the weight hyper-parameter, $\gamma$, and distance threshold, $d$, for filtering the certain target data.
The weight hyper-parameter, $\gamma$, was searched within $\{0.1, 0.5, 1.0, 2.0\}$ and $\{0.2, 0.3, 0.4, 0.5\}$ for ImageCLEF-DA and VisDA-2017 datasets, respectively.
The distance threshold hyper-parameter, $d$, was searched within $\{0.05, 0.1, 1.0\}$.
The selected hyper-parameters for each task are listed in Table~\ref{selected_hyper_parameters}.

\begin{table}[ht]
  \small
  \caption{Selected hyper-parameters for ImageCLEF-DA and VisDA-2017 experiments.}
  \label{selected_hyper_parameters}
  \centering
  \begin{tabular}{cccccccc}
    \toprule
    Parameter
    & I $\rightarrow$ P
    & P $\rightarrow$ I
    & I $\rightarrow$ C
    & C $\rightarrow$ I
    & C $\rightarrow$ P
    & P $\rightarrow$ C
    & VisDA-2017
    \\
    \midrule
    $\gamma$ & 0.5 & 1.0 & 2.0 & 2.0 & 0.1 & 1.0 & 0.3 \\
    $d$ & 0.05 & 0.1 & 0.1 & 0.1 & 1.0 & 0.1 & 0.1 \\
    \bottomrule
  \end{tabular}
\end{table}

\section{Additional implementation details}
The temperature parameter, $\tau$, for the critic function was fixed to 0.05.
For ImageCLEF-DA and VisDA-2017 datasets, the queue size, considering the dataset sizes, was set to 4,096 and 32,768, respectively, and the momentum coefficient, m, of the momentum encoder to 0.9 and 0.99, respectively.
For the metric measuring the distances in the feature space, $\mathcal{Z}$, cosine dissimilarity was applied.
At the end of the encoders, we added L2 normalization layers.
Unlike other contrastive learning methods, we did not utilize additional data augmentation for fair comparison with domain adaptation baselines;
only random crop and horizontal flip were used.
We empirically found that it is beneficial to forward pass the uncertain target data to train the batch normalization layers.
For computing infrastructure, we used a GeForce RTX 2080 Ti GPU for all experiments.
The total iterations for the ImageCLEF-DA and VisDA-2017 experiments were 20,000 and 50,000, respectively, and they took 4 h and 15 h, respectively, on an average.

\section{VisDA-2017 full results}
\begin{table}[ht]
  \small
  \caption{Accuracy (\%) on VisDA-2017 for unsupervised domain adaptation (ResNet-101).}
  \label{visda_full_result}
  \centering
  \resizebox{\textwidth}{!}{\begin{tabular}{lccccccccccccc}
    \toprule
    Method 
    & \rotatebox[origin=l]{90}{airplane} 
    & \rotatebox[origin=l]{90}{bicycle} 
    & \rotatebox[origin=l]{90}{bus} 
    & \rotatebox[origin=l]{90}{car} 
    & \rotatebox[origin=l]{90}{horse} 
    & \rotatebox[origin=l]{90}{knife} 
    & \rotatebox[origin=l]{90}{motorcycle} 
    & \rotatebox[origin=l]{90}{person} 
    & \rotatebox[origin=l]{90}{plant} 
    & \rotatebox[origin=l]{90}{skateboard} 
    & \rotatebox[origin=l]{90}{train} 
    & \rotatebox[origin=l]{90}{truck} 
    & Average 
    \\
    \midrule
    ResNet-101~\cite{he2016deep} & 72.3 & 6.1 & 63.4 & \textbf{91.7} & 52.7 & 7.9 & 80.1 & 5.6 & 90.1 & 18.5 & 78.1 & 25.9 & 49.4 \\
    DANN~\cite{JMLR:v17:15-239}     & 81.9 & 77.7 & 82.8 & 44.3 & 81.2 & 29.5 & 65.1 & 28.6 & 51.9 & 54.6 & 82.8 & 7.8 & 57.4 \\
    DAN~\cite{pmlr-v37-long15}          & 68.1 & 15.4 & 76.5 & 87.0 & 71.1 & 48.9 & 82.3 & 51.5 & 88.7 & 33.2 & \textbf{88.9} & 42.2 & 62.8 \\
    JAN~\cite{long2017deep}             & 75.7 & 18.7 & 82.3 & 86.3 & 70.2 & 56.9 & 80.5 & 53.8 & 92.5 & 32.2 & 84.5 & 54.5 & 65.7 \\
    MCD~\cite{saito2018maximum}         & 87.0 & 60.9 & 83.7 & 64.0 & 88.9 & 79.6 & 84.7 & 76.9 & 88.6 & 40.3 & 83.0 & 25.8 & 71.9 \\
    ADR~\cite{saito2018adversarial}     & 87.8 & 79.5 & 83.7 & 65.3 & 92.3 & 61.8 & 88.9 & 73.2 & 87.8 & 60.0 & 85.5 & 32.3 & 74.8 \\
    SE~\cite{french2018selfensembling}  & 95.9 & 87.4 & \textbf{85.2} & 58.6 & 96.2 & 95.7 & 90.6 & 80.0 & 94.8 & 90.8 & 88.4 & 47.9 & 84.3 \\
    CAN~\cite{kang2019contrastive}      & \textbf{97.0} & 87.2 & 82.5 & 74.3 & \textbf{97.8} & \textbf{96.2} & \textbf{90.8} & 80.7 & \textbf{96.6} & \textbf{96.3} & 87.5 & 59.9 & 87.2 \\
    \midrule
    \multirow{2}{*}{JCL} & \textbf{97.0} & \textbf{91.3} & 84.5 & 66.8 & 96.1 & 95.6 & 89.8 & \textbf{81.5} & 94.7 & 95.6 & 86.1 & \textbf{71.8} & \textbf{87.6} \\
    & $\pm$ 0.1 & $\pm$ 0.5 & $\pm$ 1.7 & $\pm$ 2.1 & $\pm$ 0.3 & $\pm$ 0.6 & $\pm$ 0.8 & $\pm$ 0.7 & $\pm$ 0.2 & $\pm$ 0.9 & $\pm$ 0.3 & $\pm$ 0.4 & $\pm$ 0.2 \\
    \bottomrule
  \end{tabular}}
\end{table}

\end{document}